\documentclass[12pt]{colt2023}

\title[Generalization Error Bounds for SGLD via Maximal Leakage]{Generalization Error Bounds for Noisy, Iterative Algorithms via Maximal Leakage}
\usepackage{times}
\usepackage{mathtools}

\usepackage{cleveref}
\coltauthor{%
 \Name{Ibrahim Issa} \Email{ibrahim.issa@aub.edu.lb}\\
 \addr American University of Beirut, Lebanon, and \'Ecole Polytechnique F\'ed\'erale de Lausanne, Switzerland
 \AND
 \Name{Amedeo Roberto Esposito} \Email{amedeoroberto.esposito@ist.ac.at}\\
 \addr Institute of Science and Technology Austria
 \AND
 \Name{Michael Gastpar} \Email{michael.gastpar@epfl.ch}\\
 \addr  \'Ecole Polytechnique F\'ed\'erale de Lausanne, Switzerland
}

\def\Pr{{\rm \mathbf {Pr}}}
\def\E{{\rm \mathbf  E}}

\def\var{{\rm \mathbf {var}}}

\newcommand{\Z}{\mathcal{Z}}
\newcommand{\B}{\mathcal{B}}
\newcommand{\R}{\mathbb{R}}
\newcommand{\W}{\mathcal{W}}
\newcommand{\dd}{\mathrm{d}}

\newcommand{\ml}[2]{\mathcal{L}\left(#1  \!\!  \to  \!\!   #2\right)} 
\newcommand{\cml}[3]{\mathcal{L}\left(#1  \!\!  \to  \!\!   #2 | #3\right)} 
\newcommand{\esssup}
{\operatornamewithlimits{ess-sup}}

\newcommand{\ds}{\displaystyle}
\newcommand{\defeq}{\vcentcolon=}

\begin{document}

\maketitle

\begin{abstract}
    We adopt an information-theoretic framework to analyze the generalization behavior of the class of iterative, noisy learning algorithms. This class is particularly suitable for study under information-theoretic metrics as the algorithms are inherently randomized, and it includes  commonly used algorithms such as Stochastic Gradient Langevin Dynamics (SGLD). Herein, we use the maximal leakage (equivalently, the Sibson mutual information of order infinity) metric, as it is simple to analyze, and it implies both bounds on the probability of having a large generalization error and on its expected value. We show that, if the update function (e.g., gradient) is bounded in $L_2$-norm and the additive noise is isotropic Gaussian noise, then one can obtain an upper-bound on maximal leakage in semi-closed form. Furthermore, we demonstrate how the assumptions on the update function affect the optimal (in the sense of minimizing the induced maximal leakage) choice of the noise. Finally, we compute explicit tight upper bounds on the induced maximal leakage for other scenarios of interest. 
    \end{abstract}


\section{Introduction}

One of the key challenges in machine learning research concerns the ``generalization'' behavior of learning algorithms. That is: if a learning algorithm performs well on the training set, what guarantees can one provide on its performance on new samples?

While the question of generalization is understood in many settings \citep{introToLearning,learningBook}, existing bounds and techniques provide vacuous expressions when employed to show the generalization capabilities of deep neural networks (DNNs) \citep{normsNN,vcDimNN, genErrDL1,genErrDL2}. In general,
classical measures of model expressivity (such as Vapnik-Chervonenkis
(VC) dimension \citep{vcDim}, Rademacher complexity \citep{rademacherComplexity}, etc.) fail to explain the generalization abilities
of DNNs due to the fact that they are typically over-parameterized models with less
training data than model parameters. 
A novel approach was introduced by \citep{explBiasMI}, and \citep{infoThGenAn} (further developed by \cite{conditionalMI,tighteningMI,fullVersionGeneralization,tpcGenerr} and many others), where information-theoretic techniques are used to link the generalization capabilities of a learning algorithm to information measures. These quantities are algorithm-dependent and can be used to analyze the generalization capabilities of general classes of updates and models \textit{e.g.}, noisy iterative algorithms such as the Stochastic Gradient Langevin Dynamics (SGLD) \citep{genErrMISGLD, optimalNoiseSGLD}, which can thus be applied to deep learning settings. Moreover, it has been shown that information-theoretic bounds can be non-vacuous and reflect the real generalization behavior even in deep
learning settings \citep{genErrNoisyIterative,nonvacuousPAC, genErrSGDLDataDependent,genErrNoisyIterativeCMI}.

In this work we adopt and expand the framework introduced by \cite{genErrMISGLD}, but instead of focusing on the mutual information between the input and output of an iterative algorithm, we compute the maximal leakage \citep{leakageLong}.
Maximal leakage, together with other information measures of the Sibson/R\'enyi family (maximal leakage can be shown to be Sibson Mutual information of order infinity \citep{leakageLong}), have been linked to high-probability bounds on the generalization error \citep{fullVersionGeneralization}. In particular, given a learning algorithm $\mathcal{A}$ trained on data-set $S$ (made of $n$ samples), one can provide the following guarantee in the case of the $0-1$ loss:
\begin{equation}
    \Pr(|\text{gen-err}(\mathcal{A},S)|\geq \eta) \leq 2\exp(-2n\eta^2 + \ml{S}{\mathcal{A}(S)}),
\end{equation}
where $\ml{S}{\mathcal{A}(S)}$ is defined in equation~\eqref{eq:Iinf_MaxLeak} below.
This deviates from much of the literature in which the focus is on bounding the \textbf{expected} generalization error instead \citep{infoThGenAn,conditionalMI}. 
Consequently, if one can guarantee that for a class of algorithms, the maximal leakage between the input and the output is bounded, then one can provide an \textbf{exponentially decaying} (in the number of samples $n$) bound on the probability of having a large generalization error. This is in general not true for mutual information, which can typically only guarantee a linearly decaying bound on the probability of the same event \citep{learningMI}. Moreover, a bound on maximal leakage implies a bound on mutual information (cf.~\Cref{eq:sibsonMI-increasing}) and, consequently, a bound on the expected generalization error of $\mathcal{A}$ (exploiting the link between mutual information and expected generalization error~\citep{infoThGenAn}). The main advantage of maximal leakage lies in the fact that it depends on the distribution of the samples only through its support. It is thus naturally independent from the distribution over the samples and particularly amenable to analysis, especially in additive noise settings.

The contributions of this work can be summarized as follows:
\begin{itemize}
    \item we derive novel bounds on $\ml{S}{\mathcal{A}(S)}$ whenever $\mathcal{A}$ is a noisy, iterative algorithm (SGLD-like), which then implies the first bounds showing generalization with high-probability of said mechanisms;
  \item we leverage the analysis to extrapolate to optimize the type of noise to be added (in the sense of minimizing the induced maximal leakage), based on the assumptions imposed on the algorithm. In particular, 
      if one assumes the $L_\infty$ norm of the gradient to be bounded, then adding uniform noise minimizes the maximal leakage upper bound.
  Hence, the analysis and computation of maximal leakage can \emph{also} be used to inform the design of novel noisy, iterative algorithms.
\end{itemize}

\subsection{Related Work}
The line of work exploiting information measures to bound the expected generalization started in~\citep{explBiasMI,infoThGenAn} and was then refined with a variety of approaches considering Conditional Mutual Information~\citep{conditionalMI,genErrNoisyIterativeCMI}, the Mutual Information between individual samples and the hypothesis~\citep{ismi} or improved versions of the original bounds~\citep{ISIT2019,improvingMICondProc}. Other approaches employed the Kullback-Leibler Divergence with a PAC-Bayesian approach~\citep{pacbayes,nonvacuousPAC}. 
Moreover, said bounds were then characterized for specific SGLD-like algorithms, denoted as ``noisy, iterative algorithms'' and used to provide novel, non-vacuous bounds for Neural Networks~\citep{genErrMISGLD,genErrSGDLDataDependent,genErrNoisyIterativeCMI, wang2023generalization} as well as for SGD algorithms~\citep{sgdBoundInfoTh}. Recent efforts tried to provide the optimal type of noise to add in said algorithms and reduce the (empirical) gap in performance between SGLD and SGD~\citep{optimalNoiseSGLD}.
All of these approaches considered the KL-Divergence or (variants of) Shannon's Mutual Information. General bounds on the expected generalization error leveraging arbitrary divergences were given in~\citep{tpcGenerr,convexAnalysisGenErr}.
Another line of work considered instead bounds on the probability of having a large generalization error~\citep{learningMI,fullVersionGeneralization,conditionalSibsonErrorBounds} and focused on large families of divergences and generalizations of the Mutual Information (in particular of the Sibson/R\'enyi-family, including conditional versions).

\section{Preliminaries, Setup, and a General Bound}
\subsection{Preliminaries}
\subsubsection{Information Measures}
The main building block of the information measures considered in this work is the R\'enyi's $\alpha$-divergence between two measures $P$ and $Q$, $D_\alpha(P\|Q)$ (which can be seen as a parametrized generalization of the Kullback Leibler-divergence)~\cite[Definition 2]{RenyiKLDiv}. Starting from R\'enyi's Divergence and the geometric averaging that it involves, Sibson built the notion of Information Radius~\citep{infoRadius} which can be seen as a special case of the following quantity~\citep{verduAlpha}: $
    I_\alpha(X,Y) = \min_{Q_Y} D_\alpha(P_{XY}\|P_{X}Q_Y). $
Sibson's $I_\alpha(X,Y)$ represents a generalization of Shannon's mutual information, indeed one has that: $
\lim_{\alpha\to 1}I_\alpha(X,Y)=I(X;Y) = \mathbb{E}_{P_{XY}}\left[\log\left(\frac{dP_{XY}}{dP_XP_Y}\right)\right].$
Differently, when $\alpha\to\infty$, one gets: \begin{equation} \label{eq:Iinf_MaxLeak}
I_\infty(X,Y)=\log\mathbb{E}_{P_Y}\left[ \esssup_{P_X} \frac{dP_{XY}}{dP_XP_Y} \right] =\ml{X}{Y},
\end{equation} where $\ml{X}{Y}$ denotes the maximal leakage from $X$ to $Y$, a recently defined information measure with an operational meaning in the context of privacy and security~\citep{leakageLong}. Maximal leakage represents the main quantity of interest for the scope of this paper, as  it is amenable to analysis and has been used to bound the generalization error~\citep{fullVersionGeneralization}. As such, we will bound the maximal leakage between the input and output of generic noisy iterative algorithms.

To that end, we mention a few useful properties of $\ml{X}{Y}$. If $X$ and $Y$ are jointly continuous random variables, then~\citep[Corollary 4]{leakageLong}
\begin{align}
    \ml{X}{Y} = \log \int \esssup_{P_X} f_{Y|X} (y|x) dy,
\end{align}
where $f_{Y|X}$ is the conditional pdf of $Y$ given $X$. Moreover, maximal leakage satisfies the following chain rule (the proof of which is given in Appendix~\ref{app:lem-chain-rule}):
\begin{lemma} \label{lem:chain-rule-maxLeak}
    Given a triple of random variables $(X,Y_1,Y_2)$, then
    \begin{align}
    \label{eq:lem-chain-rule}
    \ml{X}{Y_1, Y_2} \leq \ml{X}{Y_1} + \cml{X}{Y_2}{Y_1},
\end{align}
where the conditional maximal leakage $\cml{X}{Y_2}{Y_1}  = \esssup_{P_{Y_1}} \cml{X}{Y_2}{Y_1=y_1}$, where the latter term is interpreted as the maximal leakage from $X$ to $Y_2$ with respect to the distribution $P_{XY_2 | Y_1 = y_1}$. Consequently, for random variables $(X, (Y_i)_{i=1}^n)$,
\begin{align} \label{eq:lem:chain-rule-general}
    \ml{X}{Y^n} \leq \sum_{i=1}^n \cml{X}{Y_i}{Y^{i-1}}.
\end{align}
\end{lemma}
Moreover, one can relate $\ml{X}{Y}$ to $I(X;Y)$ through $I_\alpha$. Indeed, an important property of $I_\alpha$ is that it is non-decreasing in $\alpha$, hence for every $\infty>\alpha>1$:
\begin{equation} \label{eq:sibsonMI-increasing}
    I(X;Y) = I_1(X,Y) \leq I_\alpha(X,Y) \leq I_\infty(X,Y) = \ml{X}{Y}. 
\end{equation}
For more details on Sibson's $\alpha$-MI we refer the reader to~\citep{verduAlpha}, as for maximal leakage the reader is referred to~\citep{leakageLong}.

\subsubsection{Learning Setting}

Let $\Z$ be the sample space, $\W$ be the hypothesis space, and $\ell : \W \times \Z \rightarrow \mathbb{R}_+$ be a loss function. Say $\W \subseteq \mathbb{R}^d$. Let $S=(Z_1, Z_2, \ldots, Z_n)$ consist of $n$ i.i.d samples, where $Z_i \sim P$, with $P$ unknown. 
A learning algorithm $\mathcal{A}$ is a mapping $\mathcal{A}:\Z^n\to \W$ that given a sample $S$ provides a hypothesis $W=\mathcal{A}(S)$. $\mathcal{A}$ can be either a deterministic or a randomized mapping and undertaking a probabilistic (and information-theoretic) approach one can then equivalently consider $\mathcal{A}$ as a family of conditional probability distributions $P_{W|S=s}$ for $s\in \Z^n$ \textit{i.e.}, an information channel. Given a hypothesis $w\in\W$ the true risk of $w$ is denoted as follows:
\begin{equation}
    L_{P_Z}(w) = \mathbb{E}_{P}[\ell(w,Z)]
\end{equation}
while the empirical risk of $w$ on $S$ is denoted as follows:
\begin{equation}
    L_{S}(w) = \frac1n\sum_{i=1}^n \ell(w,Z_i).
\end{equation}
Given a learning algorithm $\mathcal{A}$, one can then define its generalization error as follows:
\begin{equation}
\text{gen-err}_\mathcal{P}(\mathcal{A},S)=L_{\mathcal{P}}(\mathcal{A}(S))-L_{S}(\mathcal{A}(S)).
\end{equation}
Since both $S$ and $\mathcal{A}$ can be random, $\text{gen-err}_\mathcal{P}(\mathcal{A},S)$ is a random variable and one can then study its expected value or its behavior in probability. Bounds on the expected value of the generalization error in terms of information measures are given in~\cite{infoThGenAn,ISIT2019,ismi,conditionalMI} stating different variants of the following bound~\citep[Theorem 1]{infoThGenAn}:
if $\ell(w,Z)$ is $\sigma^2$-sub-Gaussian\footnote{A $0$-mean random variable $X$ is said to be $\sigma^2$-sub-Gaussian if $\log\mathbb{E}[\exp(\lambda X)]\leq \sigma^2\lambda^2/2$ for every $\lambda\in\mathbb{R}$.} then
\begin{equation} \label{eq:gen-bound-MI}
    \left|\mathbb{E}[\text{gen-err}_\mathcal{P}(\mathcal{A},S)] \right| \leq \sqrt{\frac{2\sigma^2 I(S;\mathcal{A}(S))}{n}}.
\end{equation}
Thus, if one can prove that the mutual information between the input and output of a learning algorithm $\mathcal{A}$ trained on $S$ is bounded (ideally, growing less than linearly in $n$) then the expected generalization error of $\mathcal{A}$ will vanish with the number of samples.
Alternatively,~\cite{fullVersionGeneralization} demonstrate high-probability bounds, involving different families of information measures. One such bound, which is relevant to the scope of this paper is the following  
\citep[Corollary 2]{fullVersionGeneralization}: assume $\ell(w,Z)$ is $\sigma^2$-sub-Gaussian and let $\alpha>1$, then
\begin{equation}
    \Pr(|\text{gen-err}_P(\mathcal{A},S)|\geq t) \leq 2\exp\left(-\frac{\alpha-1}{\alpha}\left(\frac{nt^2}{2\sigma^2}- I_\alpha(S,\mathcal{A}(S))\right) \right) \label{eq:genErrProbBound},
\end{equation}
taking the limit of $\alpha\to\infty$ in~\eqref{eq:genErrProbBound} leads to the following~\cite[Corollary 4]{fullVersionGeneralization}:
\begin{equation}
    \Pr(|\text{gen-err}_P(\mathcal{A},S)|\geq t) \leq 2\exp\left(-\left(\frac{nt^2}{2\sigma^2}- \ml{S}{\mathcal{A}(S)}\right)\right) \label{eq:genErrProbBoundMaxLeak}.
\end{equation}
Thus, in this case, if one can prove that the maximal leakage between the input and output of a learning algorithm $\mathcal{A}$ trained on $S$ is bounded, then the \textbf{probability} of the generalization error of $\mathcal{A}$ being larger than any constant $t$ will decay \textbf{exponentially fast} in the number of samples $n$.
\subsection{Problem Setup}
We consider iterative algorithms, where each update is of the following form:
\begin{align} 
W_t = g(W_{t-1}) - \eta_t F(W_{t-1},Z_t) + \xi_t, ~ \forall ~ t \geq 1, \label{eq:iteration1}
\end{align}
where $Z_t \subseteq S$ (sampled according to some distribution), $g: \mathbb{R}^{d} \rightarrow \mathbb{R}^{d}$ is a deterministic function, $F(W_{t-1},Z_t)$ computes a direction (e.g., gradient), $\eta_t$ is the step-size, and $\xi_t =( \xi_{t1}, \ldots, \xi_{td})$ is random noise. We will assume for the remainder of this paper that $\xi_t$ has an absolutely continuous distribution.
Let $T$ denote the total number of iterations, $W^t = (W_1, W_2, \ldots W_t)$, and $Z^t = (Z_1, Z_2, \ldots, Z_t)$. The algorithms under consideration further satisfy the following two assumptions
\begin{itemize}
\item {\bf Assumption 1 (Sampling):} The sampling strategy is agnostic to parameter vectors: 
\begin{align}
P(Z_{t+1}|Z^t, W^t, S) = P(Z_{t+1} |Z^t, S).
\end{align}
\item {\bf Assumption 2 ($\mathbf{L_p}$-Boundedness):} For some $p>0$ and $L>0$, \,$\sup_{w,z} \| F(w,z) \|_p \leq L$.
\end{itemize}
As a consequence of the first assumption and the structure of the iterates, we get:
\begin{align} \label{eq:markov}
P(W_{t+1} | W^t, Z^T, S) = P(W_{t+1} | W_t, Z_{t+1}).
\end{align}
The above setup was proposed by~\cite{genErrMISGLD}, who specifically studied the case $p=2$. 
Denoting by $W$ the final output of the algorithm (some function of $W^T$), they show that
\begin{theorem}[{\cite[Theorem 1]{genErrMISGLD}}] \label{thm:Pensia}
If the boundedness assumption holds for $p=2$ and 
$\xi_t \sim \mathcal{N}(0, \sigma^2_t I_d)$, then
\begin{align} \label{eq:thm-Pensia}
I(S;W) \leq \frac{d}{2} \sum_{t=1}^T \log \left(1 + \frac{\eta^2_t L^2}{d \sigma^2_t} \right).
\end{align}
\end{theorem}
By virtue of inequality~\eqref{eq:gen-bound-MI}, this yields a bound on the expected generalization error.

In this work, we derive bounds on the maximal leakage between $\ml{S}{W}$ for iterative noisy algorithms, which leads to high-probability bounds on the generalization error (cf. equation~\eqref{eq:genErrProbBoundMaxLeak}). We consider different scenarios in which $F$ is bounded in $L_1$, $L_2$, or $L_\infty$ norm, and the added noise is Laplace, Gaussian, or Uniform. It is worth noting that the bounds we derive depend on $F$ only through the boundedness assumption (Assumption 2 above). Considering $F$ to be a gradient yields the most (practically) interesting scenario in which our results hold, as it represents a widely used family of learning algorithms. However, we do not leverage any structure that is particular to gradients (beyond the boundedness assumption).

\subsection{Notation} \label{sec:notation}
Given $d \in \mathbb{N}$, $w \in \mathbb{R}^d$, and $r >0$, let $\B_p^{d}(w,r) = \{ x \in \mathbb{R}^d: \| x - w \|_p \leq r \}$ denote the $L_p$-ball of radius $r$ and center $w$, and let $V_p (d,r)$ denote its corresponding volume. When the dimension $d$ is clear from the context, we may drop the superscript and write $\B_p(w,r)$. Given a set $S$, we denote its complement by $\overline{S}$. The $i$-th component of $w_t$ will be denoted by $w_{ti}$. 

We denote the pdf of the noise $\xi_t$ by $f_t: \R^d \rightarrow \R$. The following functional will be useful for our study: given $d \in \mathbb{N}$, $p > 0$, a pdf $f: \R^d \rightarrow \R$, and an $r \geq 0$, define
\begin{align}
    \label{eq:def-h}
    h(d,p,f,r) \defeq  \int \limits_{\overline{\B_p^d}(0, r)} \sup_{x \in \B_p^d(0,r)} f(w - x) \dd w. 
\end{align}
We denote the ``positive octant'' by $A_d$, i.e.,
\begin{align}
    \label{eq:def-positive-octant}
    A_d \defeq \{ w \in \R^d: w_i \geq 0, \text{ for all } i \in \{1,2,\ldots,d\} \}.
\end{align}
Since we will mainly consider pdfs that are symmetric (Gaussian, Laplace, uniform), the $h$ functional ``restricted'' to $A_d$ will be useful:
\begin{align}
    \label{eq:def-h-positive}
    h_+ (d,p,f,r) \defeq  \int \limits_{\overline{\B_p^d}(0, r) \cap A_d } \sup_{x \in \B_p^d(0,r)} f(w - x) \dd w. 
\end{align}

\subsection{General Bound}

\begin{proposition} \label{prop:general}
    Suppose $f_{t}: \R^d \rightarrow \R$ is maximized for $x = 0$. If Assumptions 1 and 2 hold for some $p > 0$, then
\begin{align}
    \label{eq:prop-general}
    \ml{S}{W} \leq \sum_{t=1}^T \log \left( f_{t} (0) V_p(d,\eta_t L) + h(d,p,f_t, \eta_t L) \right),
\end{align}
where $h$ is defined in equation~\eqref{eq:def-h}.
\end{proposition}

The above bound is appealing as it implicitly poses an optimization problem: given a constraint on the noise pdf $f_t$ (say, a bounded variance), one may choose $f_t$ as to minimize the upper bound in equation~\eqref{eq:prop-general}. Moreover, despite its generality, we show that it is tight in several interesting cases, including when $p=2$ and $f_t$ is the Gaussian pdf.

In the next section, we consider several scenarios for different values of $p$ and different noise distributions. As a testament to the tractability of maximal leakage, we derive exact semi-closed form expressions for the bound of Proposition~\ref{prop:general}. Finally, it is worth noting that the form of the bound allows us to choose different noise distributions at different time steps, but these examples are outside the scope of this paper. \\

\begin{proof}
We proceed as in the work of \cite{genErrMISGLD}: 
\begin{align}
\ml{S}{W} \leq \ml{Z^T}{W^T} \leq \sum_{t=1}^T \cml{Z^T}{W_t}{W^{t-1}} = \sum_{t=1}^T \cml{Z_t}{W_t}{W_{t-1}},
\end{align}
where the first inequality follows from Lemma 2 of~\cite{genErrMISGLD} and the data processing inequality for maximal leakage~\cite[Lemma 1]{leakageLong}, the second inequality follows Lemma~\ref{lem:chain-rule-maxLeak}, and the equality follows from~\eqref{eq:markov}. Now,
\begin{align}
\exp\left\lbrace \cml{Z_t}{W_t}{W_{t-1}=w_{t-1}} \right\rbrace
    & =     \int_{\R^d} \esssup_{P_{z_t}} p(w_t | Z_t) \dd w_t \\
    & = \int_{\R^d} \esssup_{P_{z_t}} f_t \left( w_t - g(w_{t-1}) + \eta_t F(w_{t-1},Z_t) \right) \dd w_t, \\
    & = \int_{\R^d} \esssup_{P_{z_t}} f_t \left( w_t + \eta_t F(w_{t-1},Z_t) \right) \dd w_t, 
\end{align}
where the last equality follows from a change of a variable ${w_t} \gets w_t - g(w_{t-1})$. Finally, since $\eta_t F(w_{t-1},z_t) \in \B_p(0, \eta_t L)$ by assumption, we can further upper-bound the above by:
\begin{align}
     & \exp\left\lbrace  \cml{Z_t}{W_t}{W_{t-1}=w_{t-1}} \right\rbrace \\
     & \leq \int_{\R^d} \sup\limits_{x_t \in \B_p(0,\eta_t L)} f_t \left( w_t + x_t \right) \dd w_t  \label{eq:prop-ball-Lp} \\
     & = \int\limits_{\B_p(0,\eta_t L) } \sup\limits_{x_t \in \B_p(0,\eta_t L)}  f_{t} \left( w_{t} + x_{t} \right) \dd w_t + \int\limits_{\overline{\B_p}(0,\eta_t L)} \sup\limits_{ x_t \in \B_p(0,\eta_t L)} f_t \left( w_t + x_t \right) \dd w_t \\
     &  = f_{t} (0) V_p(d, \eta_t L) + \int\limits_{ \overline{\B_p} (0,\eta_t L)} \sup\limits_{ x_t \in \B_p(0,\eta_t L)} f_t \left( w_t - x_t \right) \dd w_t,
\end{align}
where the last equality follows from the assumptions on $f_{t}$.
\end{proof}

\section{Boundedness in $L_2$-Norm}\label{sec:l2Norm}

Considering the case where $F$ computes a gradient, then boundedness in $L_2$-norm is a common assumption. It is commonly enforced, for instance, using gradient clipping~\citep{abadi2016tensorflow,abadi2016deep,chen2020understanding}.


\begin{theorem} \label{thm:gaussian-ml}
If the boundedness assumption holds for $p \leq 2$ and 
$\xi_t \sim \mathcal{N}(0, \sigma^2_t I_d)$, then
\begin{align} \label{eq:thm-gaussian-ml}
\ml{S}{W} \leq  \sum_{t=1}^T \log \left( \frac{V_2 (d,\eta_t L)}{(2 \pi \sigma_t^2)^{d/2}} +  \frac{1}{\Gamma \left( \frac{d}{2} \right)  } \sum_{i=0}^{d-1} {d-1\choose{i}} \Gamma \left( \frac{i+1}{2} \right)  
\left(  \frac{\eta_t L}{ \sigma_t \sqrt{2}}  \right)^{d-1-i} \right), 
\end{align} 
where $\ds V_2(d,r) = \frac{\pi^{d/2}}{\Gamma \left(\frac{d}{2}+1 \right)} r^d$. 
\end{theorem}

Note that even if the parameter $L$ is large (e.g., Lipschitz constant of a neural network~\citep{genErrSGDLDataDependent}), it appears in~\eqref{eq:thm-gaussian-ml} normalized by $\Gamma(d/2)$ so its effect is significantly dampened (as $d$ is also typically very large). 

Finally, note that the bound in Proposition~\ref{prop:general} is increasing in $p$: this can be seen from line~\eqref{eq:prop-ball-Lp}, where the supremum over $\B_p$ can be further upper-bounded by a supremum over $\B_{p'}$ for $p' > p$. Therefore for $q \leq p$, the bound induced by Proposition~\ref{prop:general} is smaller. The bound in~\Cref{thm:gaussian-ml} corresponds to $p=2$ and goes to 0 (as $d$ grows), hence the bound induced by Proposition~\ref{prop:general} goes to 0 for all $q \leq p=2$. \\

\begin{proof}
The conditions of Proposition~\ref{prop:general} are satisfied, thus it is sufficient to prove the bound for $p=2$ (cf. discussion above):
\begin{align}
    \ml{S}{W} & \leq \sum_{t=1}^T \log \left( f_{t} (0) V_2(d,\eta_t L) + \int \limits_{\overline{\B_2}(0, \eta_t L)} \sup_{x_t \in \B_2(0,\eta_{t} L)} f_t(w_t - x_t) \dd w_t \right) \\
    & = \sum_{t=1}^T \log \!\! \left( \!\! \frac{V_2(d,\eta_t L)}{(2 \pi \sigma_t^2)^{\frac{d}{2}}} + \!\!\!\! \int  \limits_{\overline{\B_2} (0, \eta_t L)} \!\! \sup_{x_t \in \B_2(0,\eta_{t} L)} \!\! \frac{1}{(2 \pi \sigma_t^2)^{\frac{d}{2}}} \exp \left\lbrace -\frac{\|w_t-x_t\|^2_2}{ 2 \sigma_t^2} \right\rbrace \dd w_t \!\! \right) \! . \label{eq:prop-for-gaussian}
\end{align}
Hence, it remains to show that the second term inside the $\log$ matches that of equation~\eqref{eq:thm-gaussian-ml}. To that end, note that the point in $\mathcal{B}_2(0,\eta_t L)$ that minimizes the distance to $w_t$ is given $\frac{\eta_t L}{\| w_t \|} w_t$. So we get
\begin{align}
\| w_t - x_t \| \geq \| w_t - \frac{\eta_t L}{\| w_t \|} w_t \| = \| w_t\| - \eta_t L.
\end{align}
 Then, 
\begin{align} 
h(d,2,f_t,\eta_t L) & =     \int \limits_{\overline{\B_2} (0, \eta_t L)} \sup_{x_t \in \B_2(0,\eta_{t} L)} \frac{1}{(2 \pi \sigma_t^2)^{\frac{d}{2}}} \exp \left\lbrace -\frac{\|w_t-x_t\|^2_2}{ 2 \sigma_t^2} \right\rbrace \dd w_t \\
&  = \int\limits_{\overline{\B_2} (0, \eta_t L)}  \frac{1}{(2 \pi \sigma_t^2)^{\frac{d}{2}}} \exp \left\lbrace - \frac{(\|w_t\|_2 - \eta_t L )^2}{2 \sigma_t^2} \right\rbrace \dd w_t.
\label{eq:almost-there}
\end{align}
To evaluate this integral, we use spherical coordinates (details in Appendix~\ref{app:spherical}). Then,
\begin{align} \label{eq:MlGaussInt}
h(d,2,f_t,\eta_t L) = \left ( \frac{\eta_t L}{ \sigma_t \sqrt{2}} \right)^{d-1} \frac{1}{\Gamma \left( \frac{d}{2} \right)  } \sum_{i=0}^{d-1} {d-1 \choose i} \left( \frac{ \sigma_t \sqrt{2}}{\eta_t L}  \right)^i \Gamma \left( \frac{i+1}{2} \right).
\end{align}
Combining equations~\eqref{eq:prop-for-gaussian} and~\eqref{eq:MlGaussInt} yields~\eqref{eq:thm-gaussian-ml}.
\end{proof}

\begin{remark}
One could also derive a semi-closed form bound for the case in which the added noise is uniform. 
\end{remark}

\section{Boundedness in $L_\infty$-Norm}

The bound in Proposition~\ref{prop:general} makes minimal assumptions about the pdf $f_t$. In many practical scenarios we have more structure we could leverage. In particular, we make the following standard assumptions in this section: 
\begin{itemize}
    \item $\xi_t$ is composed of i.i.d components. Let $f_{t0}$ be the pdf of a component, then $\ds f_t(x_t) \ =  \prod_{i=1}^d f_{t0}(x_{ti})$.
    \item $f_{t0}$ is symmetric around 0 and non-increasing over $[0,\infty).$ 
\end{itemize}
In this setting, Proposition~\ref{prop:general} reduces to a very simple form for $p = \infty$:

\begin{theorem} \label{thm:Linf}
Suppose $f_{t}$ satisfies the above assumptions. If Assumptions 1 and 2 hold for $p = \infty$, then
\begin{align} \label{eq:thm-Linf}
\ml{S}{W} \leq \sum_{t=1}^T d \log \left( 1 + 2 \eta_t L f_{t0}(0) \right).
\end{align}
\end{theorem}

Note that the bounded-$L_\infty$ assumption is \textit{weaker} than the bounded $L_2$-norm assumption. Moreover, the assumption of having a bounded $L_\infty$-norm is satisfied in~\cite{pichapati2019adaclip} where the authors clipped the gradient in terms of the $L_\infty$-norm, thus ``enforcing'' the assumption. On the other hand, the theorem has an intriguing form as, under standard assumptions, the bound depends on $f_{t0}$ only through $f_{t0}(0)$. This naturally leads to an optimization problem: given a certain constraint on the noise, which distribution $f^\star$ minimizes $f(0)$? The following theorem shows that, if the noise is required to have a bounded variance, then $f^\star$ corresponds to the uniform distribution:

\begin{theorem} \label{thm:opt-noise}
 Let $\mathcal{F}$ be the family of probability densities (over $\mathbb{R}$) satisfying for each $f \in \mathcal{F}$: 
 \begin{enumerate}
 \item $f$ is symmetric around 0.
 \item $f$ is non-increasing over $[0, \infty)$.
 \item $\E_f [X^2] \leq \sigma^2$.
 \end{enumerate}
Then, the distribution minimizing $f(0)$ over $\mathcal{F}$ is the uniform distribution $\mathcal{U}( -\sigma \sqrt{3}, \sigma \sqrt{3})$.  
 \end{theorem}

That is, uniform noise is optimal in the sense that it minimizes the upper bound in Theorem~\ref{thm:Linf} under bounded variance constraints. The proof of Theorem~\ref{thm:opt-noise} is deferred to Appendix~\ref{app:opt-noise}.

\subsection{Proof of Theorem~\ref{thm:Linf}}
    Since the assumptions of Proposition~\ref{prop:general} hold, then
    \begin{align}
        \ml{S}{W} & \leq \sum_{t=1}^T \log \left( f_t(0) V_\infty(d,\eta_t L) + \int \limits_{\overline{\B_\infty}(0, \eta_t L)} \sup_{x_t \in \B_\infty(0,\eta_{t} L)} f_t(w_t - x_t) \dd w_t \right) \\
        & = \sum_{t=1}^T \log \left( (2\eta_t L f_{t0}(0))^d  +  
        \int \limits_{\overline{\B_\infty}(0, \eta_t L)} \prod_{i=1}^d \sup_{ x_{ti}: |x_{ti}| \leq \eta_t L} f_{t0} (w_{ti}-x_{ti}) \dd w_t
        \right).
    \end{align}
It remains to show that $h(d,\infty,f_t,\eta_t L)$ (i.e., the second term inside the $\log$ in~\Cref{eq:def-h}) is equal to $(1+2\eta_t L f_{t0}(0))^d -(2\eta_t L f_{t0}(0))^d $. We will derive a recurrence relation for $h$ in terms of $d$. To simplify the notation, we drop the subscript $t$ and ignore the dependence of $h$ on $p=\infty$, $f_t$, and $\eta_t L$, so that we simply write $h(d)$ (and correspondingly, $h_+(d)$, cf.~\Cref{eq:def-h-positive}).

By symmetry, $h(d) = 2^d h_+(d)$. Letting $w^{d-1}:=(w_1,\ldots,w_{d-1})$, we will decompose the integral over $\overline{\B_\infty^d}(0,\eta_t L)$ into two disjoint subsets: 1) $ w^{d-1} \notin \B_\infty^{d-1}(0,\eta_t L)$, in which case $w_d$ can take any value in $\R$, and 2) $w^{d-1} \in \B_\infty^{d-1}(0,\eta_t L)$, in which case $w_d$ must satisfy $|w_d| > \eta_t L$.
\begin{align}
    h_+(d) &  = 
    \int \limits_{\overline{\B^{d-1}_\infty}(0, \eta_t L) \cap A_{d-1} } \prod_{i=1}^{d-1} \sup_{ x_{i}: |x_{i}| \leq \eta_t L} f (w_{i}-x_{i})  \int_{0}^\infty \sup_{x_d: |x_d| \leq \eta_t L} f(w_d - x_d) \dd w_d \dd w^{d-1}  \label{eq:Linf-messy-1}\\
    & ~ + \int \limits_{{\B^{d-1}_\infty}(0, \eta_t L) \cap A_{d-1} } \prod_{i=1}^{d-1} \sup_{ x_{i}: |x_{i}| \leq \eta_t L} f (w_{i}-x_{i})  \int_{\eta_t L}^\infty \sup_{x_d: |x_d| \leq \eta_t L} f(w_d - x_d) \dd w_d \dd w^{d-1} \label{eq:Linf-messy-2}
\end{align}
The innermost integral of line~\eqref{eq:Linf-messy-2} is independent of $w^{d-1}$ so that the outer integral is equal to $h_{+}(d-1)$. Similarly, the innermost integral of line~\eqref{eq:Linf-messy-1} is independent of $w^{d-1}$, and the supremum in the outer integral yields $f(0)$ for every $i$. Hence, we get
\begin{align}
    h(d) = \left(1+2\eta_t L f(0) \right) h (d-1)+ (2 \eta_t L f(0))^{d-1}, \label{eq:recurrence-Linf}
\end{align}
the detailed proof of which is deferred to Appendix~\ref{app:recurrence-Linf}. 
Finally, it is straightforward to check that $h(1) =1$, hence $h(d) = (1+2\eta_t L f(0))^d-(2\eta_t L f(0))^d$.

\section{Boundedness in $L_1$-Norm}

In this section, we consider the setting where Assumption 2 holds for $p=1$. By Proposition~\ref{prop:general}, any bound derived for $p=2$ holds for $p=1$ as well; in particular, Theorem~\ref{thm:gaussian-ml} applies. 
Nevertheless, it is possible to compute a semi-closed form directly for $p=1$ (cf. Theorem~\ref{thm:l1-gaussian} below).

We also consider the case in which the additive noise is Laplace, i.e., ``matching'' the $L_1$ constraint on the update function. Interestingly, we show that in this case the limit of maximal leakage, as $d$ goes to infinity, is finite.

\subsection{Bound for Laplace noise}
We say $X$ has a Laplace distribution, denoted by $X \sim \mathrm{Lap}(\mu,1/\lambda)$, if its pdf is given by $f(x)= \frac{\lambda}{2}e^{-\lambda |x-\mu|}$ for $x \in \R$, for some $\mu \in \R$ and $\lambda >0$. The corresponding variance is given by $2/\lambda^2$.
\begin{theorem} \label{thm:l1-laplace}
If the boundedness assumptions holds for $p=1$ and 
$\xi_t$ is composed of i.i.d components, each of which is $\sim \mathrm{Lap}(0, \frac{\sqrt{2}}{\sigma_t})$, then
\begin{align} \label{eq:thm-l1-laplace}
\ml{S}{W} \leq  \sum_{t=1}^T \log \left( \frac{V_1 (d,\eta_t L)}{( \sigma_t \sqrt{2})^{d}} +  \sum_{i=0}^{d-1} \frac{(\sqrt{2} \eta_t L/ \sigma_t)^i}{ i ! }  \right), 
\end{align} 
where $\ds V_1(d,r)  = \frac{(2r)^d}{ d!}$. Consequently, for fixed $T$,
\begin{align} \label{eq:thm-l1-laplace-limit}
\lim_{d \rightarrow \infty} \ml{S}{W} \leq \sum_{t=1}^T \frac{\sqrt{2} \eta_t L}{ \sigma_t}.
\end{align}
\end{theorem}

\begin{proof}
    We give a high-level description of the proof (as similar techniques have been used in proofs of earlier theorems) and defer the details to Appendix~\ref{app:l1-laplace}.  Since the multivariate Laplace distribution (for i.i.d variables) depends on the $L_1$-norm of the corresponding vector of variables, we need to solve the following problem: given $R >0$ and $w \notin \B_1(0,R)$, compute
    \begin{align} \label{eq:to-solve-B1-L1}
        \inf_{x \in \B_1(0,R)} \| w- x \|_1.
    \end{align}
The closest element in $\B_1(0,R)$ will lie on the hyperplane defining $\B_1$ that is in the same octant as $w$, so the problem reduces to projecting a point on a hyperplane in $L_1$-distance (the proof in the appendix does not follow this argument but arrives at the same conclusion). Then, we need to compute $h(d,1,f_t,\eta_t L)$. We use a similar approach as in the proof of Theorem~\ref{thm:Linf}, that is, we split the integral and derive a recurrence relation.
\end{proof}

\subsection{Bound for Gaussian noise}

Finally, we derive a bound on the induced leakage when the added noise is Gaussian:

\begin{theorem} \label{thm:l1-gaussian}
If the boundedness assumptions holds for $p=1$ and 
$\xi_t \sim \mathcal{N}(0, \sigma^2_t I_d)$, then
\begin{align} \label{eq:thm-l1-gaussian}
     \ml{S}{W} \leq  \sum_{t=1}^T \log \left( \frac{V_1(d,R_t)}{  (2 \pi \sigma^2)^{\frac{d}{2}}}+\frac{(2\eta_t L)^{d-1} (\sigma_t \sqrt{2 d}) }{ (2 \pi \sigma_t^2)^{\frac{d}{2}} ((d-1)!) }   \sum_{i=0}^{d-1} {d-1 \choose i} \left( \frac{\sigma_t \sqrt{2 d}}{\eta_t L} \right)^i \Gamma\left(\frac{i+1}{2}\right) \right).
\end{align}
\end{theorem}

In order to prove~\Cref{thm:l1-gaussian} one has to solve a problem similar to the one introduced in~\Cref{thm:l1-laplace} (cf. equation~\eqref{eq:to-solve-B1-L1}). However, in this case a different norm is involved: i.e., given $R >0$ and $w \notin \B_1(0,R)$, one has to compute
\begin{align}
    \inf_{x \in \B_1(0,R)} \|w-x\|_2.
\end{align}
Again, one can argue that the point achieving the infimum lies on the hyperplane defining $\B_1$ that is in the same octant as $w$. In other words, the minimizer $x^\star$ is such that the sign of each component is the same sign as the corresponding component of $w$ (and lies on the boundary of $\B_1$). Thus, we are projecting a point on the corresponding face of the $L_1$-ball.  The length of the projection is then appropriately lower-bounded and the induced integral is solved by an opportune choice of change of variables.  The details of the proof are given in Appendix~\ref{app:l1-gaussian}.

\section*{Acknowledgment}
The work in this manuscript was supported in part by the Swiss National Science Foundation under Grant 200364 and by the University Research Board at the American University of Beirut (Beirut, Lebanon).

\bibliography{refs}

\newpage

\appendix

\section{Proof of Lemma~\ref{lem:chain-rule-maxLeak}} \label{app:lem-chain-rule}

    Recall the definition of maximal leakage and conditional maximal leakage:
    \begin{definition}[{Maximal Leakage~\cite[Definition 1]{leakageLong}}]
    Given two random variables $(X,Y)$ with joint distribution $P_{XY}$, 
    \begin{align}
        \ml{X}{Y} = \log\sup_{U: U-X-Y}  \frac{\Pr( \hat{U}(Y) =U)}{\max_u P_U(u)},
    \end{align}
    where $U$ takes values in a finite, but arbitrary, alphabet, and $\hat{U}(Y)$ is the optimal estimator (i.e., MAP) of $U$ given $Y$.
    \end{definition}
    Similarly,
    \begin{definition}[{Conditional Maximal Leakage~\cite[Definition 6]{leakageLong}}] Given three random variables $(X,Y,Z)$ with joint distribution $P_{XYZ}$, 
        \begin{align}
            \cml{X}{Y}{Z} = \log\sup_{U: U-X-Y |Z}  \frac{\Pr( \hat{U}(Y,Z) =U)}{\Pr( \hat{U}(Z) =U)},
        \end{align}
        where $U$ takes values in a finite, but arbitrary, alphabet, and $\hat{U}(Y,Z)$ and $\hat{U}(Z)$ are the optimal estimators (i.e., MAP) of $U$ given $(Y,Z)$ and $U$ given $Z$, respectively.
    \end{definition}
It then follows that
\begin{align}
    \ml{X}{Y_1,Y_2}& =  \log \sup_{U: U-X-(Y_1,Y_2)} \frac{\Pr( \hat{U}(Y_1,Y_2) =U)}{\max_u P_U(u)} \\
    & = \log \sup_{U: U-X-(Y_1,Y_2)} \frac{\Pr( \hat{U}(Y_1,Y_2) =U)}{\Pr( \hat{U}(Y_1) =U)}  \frac{\Pr(  \hat{U}(Y_1) =U)}{\max_u  P_U(u)} \\
    & \leq \log \sup_{U: U-X-(Y_1,Y_2)} \frac{\Pr( \hat{U}(Y_1,Y_2) =U)}{\Pr( \hat{U}(Y_1) =U)} \cdot \sup_{U: U-X-(Y_1,Y_2)} \frac{\Pr(  \hat{U}(Y_1) =U)}{\max_u  P_U(u)} \\
    & \leq \log \sup_{U: U-X-Y_2|Y_1} \frac{\Pr( \hat{U}(Y_1,Y_2) =U)}{\Pr( \hat{U}(Y_1) =U)} \cdot \sup_{U: U-X-Y_1} \frac{\Pr(  \hat{U}(Y_1) =U)}{\max_u  P_U(u)} \\
    & = \cml{X}{Y_2}{Y_1} + \ml{X}{Y_1},
\end{align}
    where the last inequality follows from the fact that $U-X-(Y_1,Y_2)$ implies $U-X-Y_2|Y_1$.

The fact that
\begin{align}
    \cml{X}{Y_2}{Y_1} =\esssup_{P_{Y_1}} \cml{X}{Y_2}{Y_1=y_1},
\end{align}
has been shown for discrete alphabets in Theorem 6 of~\citep{leakageLong}. The extension to continuous alphabets is similar (with integrals replacing sums, and pdfs replacing pmfs, where appropriate).

Finally, it remains to show equation~\eqref{eq:lem:chain-rule-general}. We proceed by induction. The case $n=2$ has already been shown above. Assume the inequality is true up to $n-1$ variables, then
\begin{align}
    \ml{X}{Y^n} & \leq \ml{X}{Y_1} + \esssup_{P_{Y_1}} \cml{X}{Y_2^n}{Y_1=y_1}  \\
    &\leq \ml{X}{Y_1} + \esssup_{P_{Y_1}} \sum_{i=2}^n \cml{X}{Y_i}{Y^{i-1},Y_1 = y_1} \\ 
    & = \sum_{i=1}^n \cml{X}{Y_i}{Y^{i-1}},
\end{align}
where the second inequality follows from the induction hypothesis.

\section{Proof of equation~\eqref{eq:MlGaussInt}} \label{app:spherical}

To evaluate the integral in line~\eqref{eq:almost-there}, we write it in spherical coordinates:
\begin{align}
   & h(d,2,f_t,\eta_t L) \notag  \\
   & =   \int\limits_{\overline{\B_2} (0, \eta_t L)}  \frac{1}{(2 \pi \sigma_t^2)^{\frac{d}{2}}} \exp \left\lbrace - \frac{(\|w_t\|_2 - \eta_t L )^2}{2 \sigma_t^2} \right\rbrace \dd w_t. \notag \\
    &  = \frac{1}{(2 \pi \sigma_t^2)^{\frac{d}{2}}} \int_{0}^{2 \pi} \int_{0}^{\pi} \ldots \int_{0}^{\pi} \int_{\eta_t L}^{\infty} e^{ \frac{-(\rho-\eta_t L)^2}{2 \sigma_t ^2} } \rho^{d-1} \sin^{d-2} (\phi_1) \sin^{d-3}(\phi_2) \ldots \sin(\phi_{d-2}) \dd \rho \dd \phi_1^{d-1} \notag \\
    & = \frac{2 \pi}{(2 \pi \sigma_t^2)^{\frac{d}{2}}} \left( \int_0^\pi \sin^{d-2} (\phi_1) \dd \phi_1 \right) \! \ldots \! \left( \int_0^\pi \sin (\phi_{d-2}) \dd \phi_{d-2} \right) \!\! \left( \int_{\eta_t L}^\infty e^{ \frac{-(\rho-\eta_t L)^2}{2 \sigma_t ^2} } \rho^{d-1} \dd \rho \right).  \label{eq:eval-interal-gaussian}
\end{align}
Now, note that for any $n \in \mathbb{N}$, $ \ds  \int_{0}^{\pi} \sin^n(x) \dd x = 2  \int_{0}^{\pi/2} \sin^n(x) \dd x$, and
\begin{align}
    \int_{0}^{\pi/2} \sin^n(x) \dd x \notag & 
    \stackrel{\text{(a)}} = \int_{0}^1 \frac{u^n}{\sqrt{1-u^2}} \dd u \notag \\
      & \stackrel{\text{(b)}} = \frac{1}{2} \int_{0}^1 t^{\frac{n-1}{2}} (1-t)^{-\frac{1}{2}} \dd y \notag \\
    & \stackrel{\text{(c)}} = \frac{1}{2} \mathrm{Beta} \left(\frac{n+1}{2}, \frac{1}{2} \right)  \notag \\
    & = \frac{\sqrt{\pi} \Gamma\left(\frac{n+1}{2} \right)}{2 \Gamma\left( \frac{n}{2}+1 \right) }, \label{eq:sinInt}
\end{align}
where (a) follows from the change of variable $u = \sin x$, (b) follows from the change of variable $t = u^2$, (c) follows from the definition of the Beta function: $ \ds \mathrm{Beta}(s_1,s_2) = \int_{0}^1 t^{s_1-1}(1-t)^{s_2-1}$, and the last equality is a known property of the Beta function ($\Gamma(1/2)= \sqrt{\pi}$). Consequently,
 \begin{align}
 & 2 \pi \left( \int_0^\pi \sin^{d-2} (\phi_1) \dd \phi_1 \right) \ldots \left( \int_0^\pi \sin (\phi_{d-2}) \dd \phi_{d-2} \right)  \notag \\
& = (2 \pi) \prod_{i=1}^{d-2} \frac{\sqrt{\pi} \Gamma\left(\frac{i+1}{2} \right)}{ \Gamma\left( \frac{i}{2}+1 \right) }  = (2 \pi) \pi^{\frac{d-2}{2}} \frac{\Gamma(1)}{\Gamma(d/2)} 
 = 2 \pi^{d/2} \frac{1}{\Gamma(d/2)}. \label{eq:collect-sin}
 \end{align}
To evaluate the innermost integral, the following identity will be useful:
\begin{align} 
\int_{0}^\infty x^n e^{-x^2} dx & = \frac{1}{2}
\int_{0}^\infty t^{\frac{n+1}{2}} e^{-t} dt =
\frac{\Gamma \left( \frac{n+1}{2} \right)}{2}, 
\label{eq:xneInt}
\end{align}
where the first equality follows from the change of variable $t = x^2$. Then,
\begin{align}
\int_{\eta_t L}^{\infty} e^{ \frac{-(\rho-\eta_t L)^2}{2 \sigma_t ^2} } \rho^{d-1} d \rho
& =  \int_0^\infty e^{ \frac{-\rho^2}{2 \sigma_t ^2} } (\rho+\eta_t L)^{d-1} d\rho \label{eq:innermost-integral-gaussian-1} \\
&  = \int_0^\infty \sum_{i=0}^{d-1} {d-1 \choose i} (\eta_t L)^{d-1-i} \rho^i e^{ \frac{-\rho^2}{2 \sigma_t ^2} } d\rho \\
& \stackrel{\text{(a)}} = \sum_{i=0}^{d-1} {d-1 \choose i} (\eta_t L)^{d-1-i} \int_0^\infty \left( \sigma_t \sqrt{2} \right)^{i+1} t^i e^{- t^2} d \rho \\
& \stackrel{\text{(b)}} =  (\eta_t L)^{d-1} (\sigma_t \sqrt{2}) \sum_{i=0}^{d-1} \left( \frac{\sigma_t \sqrt{2}}{\eta_t L} \right)^i \frac{\Gamma((i+1)/2)}{2}. \label{eq:innermost-integral-gaussian}
\end{align}
where (a) follows from the change of variable $t = \rho/(\sigma \sqrt{2})$, and (b) follows from~\eqref{eq:xneInt}.

Finally, combining equations~\eqref{eq:eval-interal-gaussian},~\eqref{eq:collect-sin}, and~\eqref{eq:innermost-integral-gaussian}, we get
\begin{align}
    h(d,2,f_t,\eta_t L) & = \frac{2 \pi ^{d/2}}{ (2\pi \sigma_t^2)^{\frac{d}{2} } \Gamma(d/2)} (\eta_t L)^{d-1} (\sigma_t \sqrt{2}) \sum_{i=0}^{d-1} \left( \frac{\sigma_t \sqrt{2}}{\eta_t L} \right)^i \frac{\Gamma((i+1)/2)}{2} \\
    & = \left( \frac{\eta_t L}{\sigma_t \sqrt{2}} \right)^{d-1}  \frac{1}{\Gamma(d/2)}\sum_{i=0}^{d-1} \left( \frac{\sigma_t \sqrt{2}}{\eta_t L} \right)^i \Gamma((i+1)/2).
\end{align}

\section{Proof of equation~\eqref{eq:recurrence-Linf}} \label{app:recurrence-Linf}

The innermost integral of line~\eqref{eq:Linf-messy-2} evaluates to
\begin{align}
   \int_{\eta_t L}^\infty \sup_{x_d: |x_d| \leq \eta_t L} f(w_d - x_d) \dd w_d  = 
   \int_{\eta_t L}^\infty  f(w_d - \eta_t L) \dd w_d  = \int_{0}^\infty  f(w_d ) \dd w_d = \frac{1}{2}, \label{eq:Linf-messy-2-2}
\end{align}
where the first equality follows from the monotonicity assumptions, the second from a change of variable, and the third from the symmetry assumption. Similarly, the innermost integral of line~\eqref{eq:Linf-messy-1} evaluates to
\begin{align}
    & \int_{0}^\infty \sup_{x_d: |x_d| \leq \eta_t L} f(w_d - x_d) \dd w_d  \\
    & = \int_{0}^{\eta_t L} \sup_{x_d: |x_d| \leq \eta_t L} f(w_d - x_d) \dd w_d \dd w^{d-1} + \int_{\eta_t L}^\infty \sup_{x_d: |x_d| \leq \eta_t L} f(w_d - x_d) \dd w_d \\
    & = \eta_t L f(0) + \frac{1}{2}. \label{eq:Linf-messy-1-1}
\end{align}
Combining equations~\eqref{eq:Linf-messy-2},~\eqref{eq:Linf-messy-2-2},  and~\eqref{eq:Linf-messy-1-1}, we get
\begin{align}
    h_+ (d) & = \left( \eta_t L f(0) + \frac{1}{2} \right)  \int \limits_{\overline{\B^{d-1}_\infty}(0, \eta_t L) \cap A_{d-1} } \prod_{i=1}^{d-1} \sup_{ x_{i}: |x_{i}| \leq \eta_t L} f (w_{i}-x_{i})  \dd w^{d-1} \\
    & \quad + \frac{1}{2} \int \limits_{{\B^{d-1}_\infty}(0, \eta_t L) \cap A_{d-1} } \prod_{i=1}^{d-1} \sup_{ x_{i}: |x_{i}| \leq \eta_t L} f (w_{i}-x_{i}) \dd w^{d-1} \\
    & =  \left( \eta_t L f(0) + \frac{1}{2} \right) h_+(d-1) + \frac{1}{2} (\eta_t L f(0))^{d-1},
\end{align}
where the second equality follows from the fact that $f$ is maximized at 0, and ${\B^{d-1}_\infty}(0, \eta_t L) \cap A_{d-1}$ is a $(d-1)$-dimensional hypercube of side $\eta_t L$ (with volume $(\eta_t L)^{d-1}$). Now,
\begin{align}
    h(d) = 2^d h_+ (d) = \left(1+2\eta_t L f(0) \right) h (d-1)+ (2 \eta_t L f(0))^{d-1}.
\end{align}

\section{Proof of Theorem~\ref{thm:opt-noise}} \label{app:opt-noise}

 Consider any $f \in \mathcal{F}$, and let 
 \begin{align}
 f_+(x) = \begin{cases}  
 f(x), & x \geq 0, \\
 0, & x <0,
 \end{cases}
\qquad \text{ and } \qquad
 f_-(x) = \begin{cases}  
 0, & x \geq 0, \\
 f(x), & x <0.
 \end{cases} 
 \end{align}
 Then
 \begin{align}
 \var_f(X^2) = \int_{-\infty}^{+\infty} (f_-(x) + f_+(x)) x^2 dx = \int_{0}^{\infty} 2 f_+(x) x^2 dx,
 \end{align}
 where the second equality follows from the symmetry assumption. Note that $2f_+$ is a valid probability density over $[0, \infty)$, and let $X_+ \sim f_+$. Then, by previous equation,
 \begin{align}
 \var_f(X^2) & = \E_{(2f_+)} \left[X_+^2 \right] 
 = \int_{0}^{\infty} 2x \left( 1 - \Pr(X_+ \leq x)  \right) dx \\
 & \geq \int_{0}^{1/(2f(0))} 2x \left(1-2xf(0) \right) dx 
 = \frac{1}{12f^2 (0)}.
 \end{align}
 Hence, \begin{align}
 f(0) \geq \frac{1}{2 \sqrt{3} \sqrt{\var_f(X^2)}} \geq \frac{1}{2 \sqrt{3} \sigma}, 
\end{align} 
 which is achieved by the uniform distribution $\mathcal{U} ( -\sigma \sqrt{3}, \sigma \sqrt{3})$. \hfill $\blacksquare$
 

\section{Proof of Theorem~\ref{thm:l1-laplace}} \label{app:l1-laplace}
    First, we show that the limit of the right-hand side of equation~\eqref{eq:thm-l1-laplace} is given by the right-hand side of equation~\eqref{eq:thm-l1-laplace-limit}. Note that
    \begin{align}
        \frac{V_1 (d,\eta_t L)}{( \sigma_t \sqrt{2})^{d}} = V_1 \left(d, \frac{\eta_t L}{ \sigma_t \sqrt{2}} \right) \xrightarrow{d \rightarrow \infty} 0.
    \end{align}
    On the other hand,
    \begin{align}
        \lim_{d \rightarrow \infty} \sum_{i=0}^{d-1} \frac{(\sigma_t \eta_t L /\sqrt{2})^i}{ i ! }  = \sum_{i=0}^\infty  \frac{(\sigma_t \eta_t L /\sqrt{2})^i}{ i ! } = e^{ \sigma_t \eta_t L /\sqrt{2}}.
    \end{align}
    Since $T$ is finite, the limit and the sum are interchangeable, so that the above two equations yield the desired limit.

    We now turn to the proof of inequality~\eqref{eq:thm-l1-laplace}. For notational convenience, set $\lambda_t = \frac{\sigma_t}{\sqrt{2}}$ (so that $f_{t0}(x) = \frac{\lambda_t}{2} e^{-\lambda|x|}$ for all $x \in \R$) and $R_t = \eta_t L$.
    Since the noise satisfies the assumptions of Proposition~\ref{prop:general}, we get
    \begin{align}
        \ml{S}{W} & \leq \sum_{t=1}^T \log \left(
            f_t(0) V_1(d, R_t) + \int \limits_{\overline{\B_1}(0, R_t)} \sup_{x_t \in \B_1(0,R_t)} f_t(w_t - x_t) \dd w_t 
        \right) \\
        & =  \sum_{t=1}^T \log \left( 
            \frac{V_1(d,R_t)}{ (\lambda_t/2)^d}       + \int \limits_{\overline{\B_1}(0, R_t)} \sup_{x_t \in \B_1(0,R_t)}
            \left(\frac{\lambda_t}{2} \right)^d
            \exp \left\lbrace - \lambda \|w_t - x_t \|_1 \right\rbrace \dd w_t
            \right).
    \end{align} 
    Recall $h(d,p,f_t,R_t)$ (cf. equation~\eqref{eq:def-h}) is defined to be the second term inside the $\log$.
    Similarly to the strategy adopted in the proof of Theorem~\ref{thm:Linf}, we will derive a recurrence relation for $h$ in terms of $d$, as such we will again suppress the dependence on $p$, $f_t$, and $R_t$ in the notation, and write $h(d)$ only (and correspondingly $h_+(d)$).  
    \begin{lemma}
        \label{lem:l1-l1-minimizer}
        Given $w \in \overline{\B_1^d}(0, R) \cap A_{d}$ ($A_d$ defined in equation~\eqref{eq:def-positive-octant}),
        \begin{align}
            \inf_{x \in \B_1^d(0,R)} \| w-x \|_1 = \sum_{i=1}^d w_i - R.
        \end{align}
    \end{lemma}
    \begin{proof}
Since we are minimizing a continuous function over a compact set, then the infimum can be replaced with a minimum. 
        
\emph{Claim:} There exists a minimizer $x^\star$ such that for all $i$, $x^\star_i \leq w_i$.

\emph{Proof of Claim:}  Consider any $x \in \B_1(0,R)$ such that there exists $j$ satisfying $x_j > w_j$. Note that $w_j \geq 0$ by assumption. Now define $x'=(x_1,\ldots,x_{j-1},w_j,x_{j+1}, \ldots, x_d)$. Then $\| x' \|_1 < \|x \|_1$ so that $x' \in \B_1(0,R)$. Moreover, $\|w-x'\|_1 \leq \| w-x\|_1$ as desired. \hfill $\blacksquare$

Now, \begin{align}
     \inf_{x \in \B_1^d(0,R)} \| w-x \|_1 
     =  \inf_{ \substack{x \in \B_1^d(0,R): \\
     x_i \leq w_i, ~\forall~ i} } \| w-x \|_1 
     = \inf_{ \substack{x \in \B_1^d(0,R): \\
     x_i \leq w_i, ~\forall i} } \sum_{i=1}^d (w_i -x_i) =\sum_{i=1}^d w_i- R.
\end{align}
    \end{proof}
Given the above lemma, we will derive the recurrence relation by decomposing the integral over $\overline{\B_1^d}(0,R_t)$ into two disjoint subsets: 1) $ w^{d-1} \notin \B_1^{d-1}(0,R_t)$, in which case $w_d$ can take any value in $\R$, and 2) $w^{d-1} \in \B_1^{d-1}(0,R_t)$, in which case $w_d$ must satisfy $|w_d| > R_t -\|w^{d-1}\|_1$.
\begin{align}
    h_+(d) & = \int \limits_{\overline{\B_1^d}(0, R_t) \cap A_{d}} \sup_{x_t \in \B_1(0,R_t) }
            \left(\frac{\lambda_t}{2} \right)^d
            e^{ - \lambda_t \left( \sum_{i=1}^d w_t - R_t \right)} \dd w_t \\
        & = \int \limits_{\overline{\B_1^{d-1}}(0, R_t) \cap A_{d}} \left( \frac{\lambda_t}{2} \right)^{d-1} e^{ - \lambda_t \left( \sum_{i=1}^{d-1} w_t - R_t \right)} \left( \int_0^{\infty} \frac{\lambda_t}{2} e^{-\lambda_t w_d } \dd w_d \right) \dd w^{d-1} \\
        & \quad + \int \limits_{{\B_1^{d-1}}(0, R_t) \cap A_{d}} \left( \frac{\lambda_t}{2} \right)^{d-1} e^{ - \lambda_t \left( \sum_{i=1}^{d-1} w_t - R_t \right)} \left( \int_{R_t-\sum_{i=1}^{d-1} w_i}^{\infty} \frac{\lambda_t}{2} e^{-\lambda_t w_d } \dd w_d \right) \dd w^{d-1} \\
        & = \frac{1}{2} h_+(d-1) +  \int \limits_{{\B_1^{d-1}}(0, R_t) \cap A_{d}} \left( \frac{\lambda_t}{2} \right)^{d-1} e^{ - \lambda_t \left( \sum_{i=1}^{d-1} w_t - R_t \right)} \left( \frac{1}{2} e^{-\lambda_t(R_t-\sum_{i=1}^d w_i) } \right)  \dd w^{d-1} \\
        & = \frac{1}{2} h_+(d-1) + \frac{1}{2}\left(\frac{\lambda_t}{2} \right)^{d-1} \frac{V_1(d-1,R_t)}{2^{d-1}}  \\
        & = \frac{1}{2}  h_+(d-1) + \frac{1}{2} \left( \frac{\lambda_t R_t}{2} \right)^{d-1} \frac{1}{(d-1)!} .
\end{align}
    Hence,
    \begin{align}
        h(d) = 2^d h_+(d) = h(d-1)+ \frac{(\lambda_t R_t)^{d-1}}{(d-1)!}.
    \end{align}
    It is easy check that $h(1)=1$, and hence
    \begin{align}
        h(d) = \sum_{i=0}^{d-1} \frac{(\lambda_t R_t)^i}{i!} 
    \end{align}
    satisfies the base case and the recurrence relation. Re-substituting $\eta_t L$ and $\sigma_t/\sqrt{2}$ for $R_t$ and $\lambda_t$, respectively, yields the desired result in equation~\eqref{eq:thm-l1-laplace}.

\section{Proof of Theorem~\ref{thm:l1-gaussian}} \label{app:l1-gaussian}
Let $R_t = \eta_t L$.
        Since the noise satisfies the assumptions of Proposition~\ref{prop:general}, we get
    \begin{align}
        \ml{S}{W} & \leq \sum_{t=1}^T \log \left(
            f_t(0) V_1(d, R_t) + \int \limits_{\overline{\B_1}(0, R_t)} \sup_{x_t \in \B_1(0,R_t)} f_t(w_t - x_t) \dd w_t 
        \right) \\
        & =  \sum_{t=1}^T \log \left( 
            \frac{V_1(d,R_t)}{ (2 \pi \sigma^2)^{\frac{d}{2}}} + 
            \int \limits_{\overline{\B_1}(0, R_t)} \sup_{x_t \in \B_1(0,R_t)}
            \frac{1}{ (2 \pi \sigma^2_t)^{\frac{d}{2}}}
            \exp \left\lbrace - \frac{\|w_t - x_t \|_2^2}{2 \sigma^2_t} \right\rbrace \dd w_t
            \right).
    \end{align} 
    Consider
    \begin{align}
        h_+(d) =  \int \limits_{\overline{\B_1}(0, R_t) \cap A_d} \sup_{x_t \in \B_1(0,R_t)}
            \frac{1}{ (2 \pi \sigma^2_t)^{\frac{d}{2}}}
            \exp \left\lbrace - \frac{\|w_t - x_t \|_2^2}{2 \sigma^2_t} \right\rbrace \dd w_t.
    \end{align}
    First we solve
    $ \ds
        \inf_{x_t \in \B_1(0,R_t)} \|w_t-x_t\|_2
    $. 
    If $w_t \in A_d$, then the infimum is achieved for $x^\star_t \in A_d$ as well (one can simply flip the sign of any negative component, which cannot increase the distance). In the subspace $A_d$, the boundary of the $L_1$ ball is defined by the hyperplane $\sum_{i=1}^d x_{ti} = R_t$. As such, finding the minimum distance corresponds to projecting the point $w$ to the given hyperplane:
    \begin{align}
        \inf_{x_t \in \B_1(0,R_t)} \|w_t-x_t\|_2 = \min_{ \substack{x_t \in \B_1(0,R_t) \cap A_d: \\ \sum_{i=1}^d x_i = R_t}} \|w_t-x_t\|_2 \geq \frac{\sum_{i=1}^d w_{ti} - R_t}{\sqrt{d}}.
    \end{align}
    Now,
    \begin{align}
        h_+(d) \leq  \int \limits_{\overline{\B_1}(0, R_t) \cap A_d} 
            \frac{1}{ (2 \pi \sigma^2_t)^{\frac{d}{2}}}
            \exp \left\lbrace - \frac{(\sum_{i=1}^d w_{ti} - R_t)^2}{2 d\sigma^2_t} \right\rbrace \dd w_t.
    \end{align}
    For notational convenience, we drop the $t$ subscript in the following. We perform a change of variable as follows: $\tilde{w}_d = \sum_{i=1}^{d} w_i $. Hence, for $w \notin \B_1(0,R)$, $\tilde{w}_d \geq R$. Since $w_d \geq 0$, then $\sum_{i=1}^{d-1} w_i \leq \tilde{w}_d$. For $x \in \R$, define $S(x) := \{ w^{d-1} \in \R^{d-1}: \sum_{i=1}^{d-1} w_i \leq x \}$. Then,
    \begin{align}
        h_+(d) & = \int_{R}^\infty \int\limits_{S(\tilde{w}_d)}  \frac{1}{ (2 \pi \sigma^2)^{\frac{d}{2}}} e^{ -\frac{(\tilde{w}_d-R)^2}{2d \sigma^2}} \dd w^{d-1} \dd w_d \\
        & =  \frac{1}{ (2 \pi \sigma^2_t)^{\frac{d}{2}}}  \int_{R}^\infty e^{ -\frac{(\tilde{w}_d-R)^2}{2d \sigma^2}} \left( \int\limits_{S(\tilde{w}_d)} \dd w^{d-1} \right) \dd w_d \\
        & \stackrel{\text{(a)}} = \frac{1}{ (2 \pi \sigma^2)^{\frac{d}{2}} ((d-1)!) } \int_R^\infty \tilde{w}_d^{d-1}   e^{ -\frac{(\tilde{w}_d-R)^2}{2d \sigma^2}} \dd w_d \\
        & \stackrel{\text{(b)}} =  \frac{1}{ (2 \pi \sigma^2)^{\frac{d}{2}} ((d-1)!) }  R^{d-1} (\sigma \sqrt{2 d}) \sum_{i=0}^{d-1}  {d-1 \choose i} \left( \frac{\sigma \sqrt{2 d}}{R} \right)^i \frac{\Gamma((i+1)/2)}{2},
    \end{align}
    where (a) follows from the fact that the innermost integral corresponds to the volume of a scaled probability simplex (scaled by $\tilde{w}_d$), and (b) follows from the same computations as in~\Crefrange{eq:innermost-integral-gaussian-1}{eq:innermost-integral-gaussian} (with $\tilde{\sigma} = \sigma \sqrt{d}$). Noting that $h(d) = 2^d h_+ (d)$ yields the desired the term in~\Cref{eq:thm-l1-gaussian}.

\end{document}